\documentclass[british,disable]{article}

\usepackage[T1]{fontenc}
\usepackage[latin9]{inputenc}
\usepackage{verbatim}
\usepackage{float}
\usepackage{amsthm}
\usepackage{amstext}
\usepackage{amssymb}
\usepackage[authoryear]{natbib}

\makeatletter

\newcommand{\noun}[1]{\textsc{#1}}
\floatstyle{ruled}
\newfloat{algorithm}{tbp}{loa}
\providecommand{\algorithmname}{Algorithm}
\floatname{algorithm}{\protect\algorithmname}

\usepackage[color=green!40]{todonotes}
\theoremstyle{plain}
\newtheorem{thm}{\protect\theoremname}
\theoremstyle{definition}
\newtheorem{defn}[thm]{\protect\definitionname}
\theoremstyle{plain}
\newtheorem{prop}[thm]{\protect\propositionname}
\ifx\proof\undefined
\newenvironment{proof}[1][\protect\proofname]{\par
\normalfont\topsep6\p@\@plus6\p@\relax
\trivlist
\itemindent\parindent
\item[\hskip\labelsep\scshape #1]\ignorespaces
}{%
\endtrivlist\@endpefalse
}
\providecommand{\proofname}{Proof}
\fi
\theoremstyle{remark}
\newtheorem{rem}[thm]{\protect\remarkname}
\theoremstyle{plain}
\newtheorem{cor}[thm]{\protect\corollaryname}

\usepackage{algpseudocode}

\algrenewcommand{\algorithmiccomment}[1]{\hfill{\em$\triangleright$\ #1}}

\usepackage{url}

\makeatother

\usepackage{babel}
\providecommand{\corollaryname}{Corollary}
\providecommand{\definitionname}{Definition}
\providecommand{\propositionname}{Proposition}
\providecommand{\remarkname}{Remark}
\providecommand{\theoremname}{Theorem}

\begin{document}

\title{On the cost-complexity of multi-context systems}

\author{Peter Novák and Cees Witteveen\\
	Algorithmics, EEMCS\\
	Delft University of Technology\\
	The Netherlands}
\maketitle
\begin{abstract}
Multi-context systems provide a powerful framework for modelling information-aggregation
systems featuring heterogeneous reasoning components. Their execution can,
however, incur non-negligible cost. Here, we focus on \emph{cost-complexity}
of such systems. To that end, we introduce \emph{cost-aware multi-context
systems}, an extension of non-monotonic multi-context systems framework taking
into account costs incurred by execution of semantic operators of the individual
contexts. We formulate the notion of cost-complexity for consistency and
reasoning problems in MCSs. Subsequently, we provide a series of results
related to gradually more and more constrained classes of MCSs and finally
introduce an incremental cost-reducing algorithm solving the reasoning problem
for definite MCSs.
\end{abstract}

\section{Introduction \label{sec:Introduction}}

Deployment of large-scale sensor networks and exploitation of heterogeneous
databases concentrating various kinds of information about the real world
opens new horizons in real-time information aggregation and processing systems.
Sensed information can be instantly cross-validated, fused, reasoned about
and further processed in real-time, so as to provide constant and up-to-date
situational awareness for people, systems, or organisations. Such knowledge-intensive
information-aggregation systems find applications in a range of industrial
domains, from marine traffic monitoring to applications in supporting assisted
living environments. With the grow of computing power, however, it's rather
the resource costs incurred by running such systems, which pose an obstacle
to their deployment, rather than the time-complexity of their execution.
Such costs can include network bandwidth, electricity, battery life, but
also direct financial costs of accessing 3rd party databases, or utilisation
of costly communication channels, such as satellite data-links.

Non-monotonic multi-context systems (MCS) introduced by \citet{DBLP:conf/aaai/BrewkaE07}
are a powerful framework for interlinking heterogeneous knowledge sources.
The framework traces its origins back to the seminal work by \citet{DBLP:journals/ai/GiunchigliaS94}
on multi-language hierarchical logics. A multi-context system comprises a
number of knowledge bases, contexts, each encapsulating a body of information,
together with a corresponding mechanism for its semantic interpretation and
reasoning with it. The flow of information among the knowledge bases is regulated
by a set of bridge rules of the form \emph{``if $L$ is true according to
the semantics of the context $i$ and ..., then $L^{\prime}$ needs to be
taken into account by the context $j$.''} Due to the abstraction from the
particularities of the internal semantics of the individual contexts and
the focus on analysis of the information flow between them in a rigorous
manner, multi-context systems provide a suitable abstraction for modelling
a wide range of information-aggregation systems, such as those mentioned
above. 

Consider an information-aggregation system aiming at surveillance and anomaly
detection in maritime traffic. Such a system would source a range of data
elements from a deployed large-scale sensor network including radars or antennae
and would cross-validate the information with that stored in local or remote
databases providing data about vessel types, owners, etc. Similarly, an information-aggregation
system supporting an assisted-living environment would continuously sense
data about well-being of patients from a range of sensors and fuse it with
relevant health records, etc. A typical query to such systems could aim at
detection whether a vessel, or a patient might need operator's attention,
such as whether a ship might be involved in suspicious activities, or whether
a patient is possibly in a life threatening condition. Querying such physical
information sources can, however, be relatively costly, while the time-complexity
of reasoning with such components plays a lesser role. 

We model such information-aggregation systems as multi-context systems as
follows. The contexts correspond to information-processing agents and information-source
adapters, each encapsulating a fragment of the information-fusion functionality
of the system according to some internal semantics with an attached cost.
The contexts are linked to each other by bridge rules prescribing the information
flow within the aggregation process, typically from low-level sensory evidence
and raw information retrieved from various databases to higher-level hypotheses
a user might be interested in.

To facilitate such multi-context systems, here we propose the framework of
\emph{cost-aware multi-context systems}, an extension of the generic framework
of non-monotonic multi-context systems~\citep{DBLP:conf/aaai/BrewkaE07}.
Subsequently, after introducing the notion of cost-complexity of algorithms
over MCSs, in a series of analyses we provide worst-case cost-complexity
upper bounds for problems of consistency and reasoning with general, definite
and acyclic MCS. We conclude our discourse with an algorithm for incremental
reasoning in definite MCSs, version of which we also implemented and deployed
in \noun{Metis}, a prototype system for maritime traffic surveillance~\citep{Hendriks2013542,METISECAI2014}.

\section{Cost-aware multi-context systems}

\global\long\def\kb{\mathit{kb}}

\global\long\def\bs{\mathit{bs}}

\global\long\def\br{\mathit{br}}

\global\long\def\KB{\mathsf{KB}}

\global\long\def\BS{\mathsf{BS}}

\global\long\def\ACC{\mathsf{ACC}}

\global\long\def\dnot{\mathit{not}}

\global\long\def\head#1{\mathit{head(#1)}}

\global\long\def\body#1{\mathit{body(#1)}}

\global\long\def\cost{\mathit{cost}}

\global\long\def\Cost#1{\mathit{Cost}_{#1}}

\global\long\def\cCost{\mathit{Cost}}

\global\long\def\CTime{\mathit{CTime}}

\global\long\def\Tmpl#1{\overline{#1}}

\global\long\def\In{\mathit{In}}

\global\long\def\IC{\mathit{IC}}

\global\long\def\Strat{\mathfrak{S}}

We build the framework of cost-aware multi-context systems as an extension
of the original non-monotonic multi-context systems by~\citet{DBLP:conf/aaai/BrewkaE07}.
\begin{defn}[\textbf{logic suite}]
\label{def:logic} A \emph{logic suite} $L=(\KB,\BS,\ACC,\cost)$ is composed
of the following components:
\begin{description}
\item [{$\KB$}] is the set of well-formed finite knowledge bases of $L$. We
assume that each element of $\KB$ is a finite set and that $\emptyset\in\KB$;
\item [{$\BS$}] is the set of possible finite belief sets;
\item [{$\ACC:\KB\rightarrow2^{\BS}$}] is a semantic operator which, given a
knowledge base $\kb$, returns a set of sets of acceptable beliefs, each
with cardinality polynomial in the size of $\kb$; and finally
\item [{$\cost:\KB\rightarrow\mathbb{R}^{+}$}] is a cost function assigning to
each knowledge base $\kb\in\KB$, the cost associated with a single execution
of the semantic operator $\ACC$ over $\kb$. Consequently, $\cost(\ACC)=\max_{\kb\in\KB}\cost(\kb)$
denotes the maximal cost which can be incurred by invocation of $\ACC$ over
the knowledge bases of $\KB$.
\end{description}
\end{defn}

Relative to the original formulation, the definition above introduces several
simplifications. We focus on the subclass of finite non-monotonic multi-context
systems, those with finite knowledge bases and bridge rule sets. We also
identify the acceptable belief sets returned by the semantic operator $\ACC$
with their poly-size kernels~(c.f.~\citealt{DBLP:conf/aaai/BrewkaE07}). 
\begin{defn}[\textbf{bridge rule}]
\label{def:bridge-rule} Let $L=\{L_{1},\ldots,L_{n}\}$ be a set of logic
suites. An \emph{$L_{i}$-bridge rule} over $L$ with $1\leq i\leq n$, is
of the form 
\[
s\longleftarrow(c_{1}:p_{1}),\ldots,(c_{j}:p_{j}),~\dnot~(c_{j+1}:p_{j+1}),\ldots,\dnot~(c_{m}:p_{m})
\]
 where $c_{k}=1..n$, $p_{k}\in S_{c_{k}}$ is an element of some belief
set $S_{c_{k}}\in\BS_{c_{k}}$ of $L_{c_{k}}$, and for each $\kb\in\KB_{i}$,
we have that $\kb\cup\{s\}\in\KB_{i}$.

For a bridge rule $r$ of the above form, $\head r=s$ and $\body r=\{p_{1},\ldots,p_{m}\}$
denote the head and the body of $r$. We say that literals $s$ and $p_{1},\ldots,p_{n}$
\emph{occur} in the head and the body of $r$ respectively.
\end{defn}

\begin{defn}[\textbf{multi-context system}]
\label{def:mcs} A \emph{cost-aware multi-context system }(MCS) $M=(C_{1},\ldots,C_{n})$
consists of a collection of contexts $C_{i}=(L_{i},\br_{i})$, where $L_{i}=(\KB_{i},\BS_{i},\ACC_{i},\cost_{i})$
is a logic suite and $\br_{i}$ is a set of $L_{i}$-bridge rules over $\{L_{1},\ldots,L_{n}\}$.
\end{defn}

The sets of knowledge bases and belief sets effectively determine the input/output
interface languages for a context $C_{i}$. To let a context process a new
information, a new element needs to be added to its knowledge base. Conversely,
retrieving information from a context corresponds to inspecting its belief
set. 

In contrast to the original definition, we do not require a context to have
an initial knowledge base, as such ``default'' input to the semantic operator
can be contained directly in its semantics, i.e., not necessarily $\ACC(\emptyset)\neq\emptyset$.
\begin{defn}[\textbf{notation}]
 We say that \emph{$r$ is a bridge rule of a MCS $M=(C_{1},\ldots,C_{n})$}
iff there exists $i=1..n$, such that the set of bridge rules $\br_{i}$
of the context $C_{i}$ contains $r$, i.e., $r\in\br_{i}$. We also say
that $M$ \emph{contains} $r$. Similarly, $M$ \emph{contains} a set of
bridge rules $R$ if it contains every rule $r\in R$. Finally, for convenience,
let $\mathcal{R}(M)=\bigcup_{i=1}^{n}\br_{i}$ denote the set of bridge rules
of $M$.
\end{defn}

\begin{defn}[\textbf{belief state and satisfied rules}]
 Let $M=(C_{1},\ldots,C_{n})$ be a MCS. A \emph{belief state} is a tuple
$S=(S_{1},\ldots,S_{n})$, such that each $S_{i}$ is a an element of $\BS_{i}$.
We define set operations on belief states as the corresponding set operations
on their respective belief set projections.

A bridge rule $r$ of the form introduced in Definition~\ref{def:bridge-rule}
is said to be \emph{satisfied} in a belief state $S$ iff for all $i=1..j$
we have $p_{i}\in S_{c_{i}}$ and for all $k=j+1..m$ we have $p_{k}\not\in S_{c_{k}}$.
\end{defn}

\begin{defn}[\textbf{equilibrium}]
 A belief state $S=(S_{1},\ldots,S_{n})$ of a MCS $M=(C_{1},\ldots,C_{n})$
with $C_{i}=(L_{i},\br_{i})$ is an \emph{equilibrium of $M$} iff for all
$i=1..n$ we have that $\kb_{i}=\{\head r\mid r\in\br_{i}\mbox{ is a rule satisfied in }S\}$
and for each $i$, we have $S_{i}\in\ACC_{i}(\kb_{i})$.
\end{defn}

\section{Reasoning with cost-aware MCSs}

The following definition of consistency and reasoning problems reiterates
the original one by \citeauthor{DBLP:conf/aaai/BrewkaE07}.
\begin{defn}[\textbf{consistency and reasoning problems}]
 Given a MCS $M=(C_{1},\ldots,C_{n})$, the problem of $M$'s \emph{consistency}
equals to deciding whether there exists an equilibrium $S=(S_{1},\ldots S_{n})$
of $M$.

Given an element $p$, a \emph{query}, the problem of \emph{brave reasoning}
is to decide whether there is an equilibrium $S=(S_{1},\ldots,S_{n})$ of
$M$, such that $p\in S_{i}$ for some $i=1..n$. We also say that $S$ entails
$p$. Finally, the problem of \emph{cautious reasoning} is to decide whether
all equilibria of $M$ entail $p$.
\end{defn}
Due to the opaqueness of the individual contexts in a MCS, an algorithm for
deciding a problems of consistency, brave, or cautious reasoning, would in
general need to search for a solution by testing various knowledge bases
as inputs to contexts, executing their internal semantic operators, and finally
check whether the outputs are coherent with the knowledge bases. Informally,
the cost incurred by a run of such a computation over the input MCS corresponds
to the sum of the costs associated with the series of invocations of the
semantic operators of the individual contexts.

\begin{defn}[\textbf{cost-complexity}]
 \label{def:cost-complexity} Let $\mathcal{A}$ be a \emph{deterministic}
algorithm taking as an input a MCS $M$ and computing a particular belief
state state $S$ of $M$ as its output, along the way employing the semantic
operators of the individual contexts. Given a series of semantic operator
invocations $\ACC_{c_{1}},\ldots,\ACC_{c_{m}}$ performed during $\mathcal{A}$'s
execution, $\Cost{\mathcal{A}}(M)=\sum_{i=1}^{m}\cost(\ACC_{c_{i}})$ denotes
the sum of the costs of their corresponding invocations. We also say that
$\Cost{\mathcal{A}}(M)$ is a cost-complexity of $\mathcal{A}$'s computation
on $M$.

The \emph{worst-case cost-complexity} of $\mathcal{A}$ is a function $\Cost{\mathcal{A}}:\mathbb{N}\times\mathbb{N}\rightarrow\mathbb{N}$
defined by 
\[
\Cost{\mathcal{A}}(n,m)=\max\{\Cost{\mathcal{A}}(M)\mid M\in\overline{M}_{n,m}\}
\]

\end{defn}
where $\overline{M}_{n,m}$ is a set of all MCSs composed of precisely $n$
contexts and $m$ bridge rules. That is, for each $M\in\overline{M}_{n,m}$,
we have $M=(C_{1},\ldots,C_{n})$, with each $C_{i}$ comprising bridge rules
$\br_{i}$ and $m=|\mathcal{R}(M)|=\sum_{i=1}^{n}|\br_{i}|$.

In the restricted case when the number of bridge rules $m$ in an MCS is
bounded with respect to the number of its contexts $n$ by some finite factor
$k\in\mathbb{N}$, we define $\Cost{\mathcal{A}}(n)=\Cost{\mathcal{A}}(n,k\cdot n)$.

\begin{comment}
Furthermore, for a given decision problem $\mathcal{P}$, we say that the
problem is $\cCost(n,m)$, or $\cCost(n)$ hard iff $\cCost(n,m)=\min_{\mathcal{A}}\Cost{\mathcal{A}}(n,m)$
and $\cCost(n)=\min_{\mathcal{A}}\Cost{\mathcal{A}}(n)$ respectively, where
$\mathcal{A}$ ranges over algorithms which decide $\mathcal{P}$.
\end{comment}

Consider a special class of MCSs with uniform unit cost of execution of all
semantic operators of their corresponding contexts. For such MCSs, the notion
of cost-complexity of algorithms reduces to the notion of time-complexity
in terms of the number of invocations of the context semantic operators.
\begin{defn}[\textbf{context-independent time complexity}]
 A MCS $M=(C_{1},\ldots,C_{n})$ is said to be \emph{uniform-cost} iff for
all $i=1..n$, we have $\cost(\ACC_{i})=1$, with $\ACC_{i}$ corresponding
to $C_{i}$.

The context-independent time complexity is defined as $\CTime_{\mathcal{A}}(M)=\Cost{\mathcal{A}}(M)=m$.
Consequently, the context-independent worst-case complexity of $\mathcal{A}$
is defined as $\CTime_{\mathcal{A}}(n,m)=\Cost{\mathcal{A}}(n,m)$ over uniform-cost
MCSs with $n$ contexts and $m$ bridge rules. $\CTime_{\mathcal{A}}(n)$
and $\CTime(n)$ are defined accordingly in relation to $\Cost{\mathcal{A}}(n)$
and $\cCost(n)$.
\end{defn}
Finally, we analyse the context-independent time complexity and the cost-complexity
of the class of general non-monotonic cost-aware multi-context systems.
\begin{prop}[\textbf{consistency}]
 \label{prop:problems-complexity} Given a uniform-cost MCS $M=(C_{1},\ldots,C_{n})$,
an upper bound on the worst-case context-independent time complexity of deciding
the consistency problem for $M$, as well as problems of cautious and brave
reasoning w.r.t.~$M$ for some query $p$, we have 
\[
\CTime(n,m)\leq n\cdot2^{m}
\]
where $m$ is the number of bridge rules in $M$.

In the case $M$ is not a uniform-cost MCS, an upper bound on the worst-case
cost-complexity of deciding the consistency problem for $M$, as well as
problems of cautious and brave reasoning w.r.t.~$M$ for some query $p$,
we have 
\[
\cCost(n,m)\leq c\cdot\CTime(n,m)
\]
 where $ $$c=\max_{i=1..n}\cost(\ACC_{i})$.\end{prop}
\begin{proof}
Consider the following algorithmic schema:
\begin{enumerate}
\item \label{enu:general-MCS-schema-1} guess the set of bridge rules $R$ to be
satisfied in an equilibrium;
\item construct the knowledge bases $kb_{1},\ldots\kb_{n}$, so that $\kb_{i}=\{\head r\mid r\in R\}$;
\item \label{enu:general-MCS-schema-3} execute the individual contexts' semantic
operators on the knowledge bases and thus obtain a belief state $S$; and
finally
\item check whether $S$ is an equilibrium. That is, exactly the rules from $R$
are those satisfied in $S$.
\end{enumerate}
In general, there are at most $2^{m}$ candidate sets of rules to guess in
the step~\ref{enu:general-MCS-schema-1} of the non-deterministic schema
above. For each of them, we need to invoke at most $n$ semantic operators
in the step~\ref{enu:general-MCS-schema-3}, what in turn incurs a cost
of at most $c$ per invocation.

For brave and cautious reasoning problems, in the worst case, we need to
enumerate all the possible belief states to check whether they are equilibria
and additionally whether they entail the query $p$. Hence, the worst-case
cost complexity of reasoning problem equals the one of the consistency problem.
\end{proof}
The above result is consistent with the complexity analysis of \citeauthor{DBLP:conf/aaai/BrewkaE07}.
They show that the time complexity of the decision problems as a step-up
over the time complexity of $\ACC$ operator: if the context time-complexity
is in $\Delta_{k}^{P}$, then the time-complexity of deciding the consistency
problem lies in $\Sigma_{k+1}^{P}$. We abstract away from the context time-complexity
and consider it constant ($P$), hence the expected step-up corresponds to
$\mathit{NP}$.

\section{Definite cost-aware MCS \label{sec:Definite-cost-aware-MCS}}

The cost-complexity characteristics of reasoning with general cost-aware
multi-context systems as introduced in the previous section is rather pessimistic.
Even brave reasoning incurs in general cost complexity exponential in the
size of the information-flow structure of the system. For practical purposes,
that can become prohibitive as the size of the system scales. Often, however,
information flows of implemented systems feature simpler structure both in
terms of the individual contexts, as well as in terms of the underlying flow
of information. \emph{Definite cost-aware multi-context systems}, an adaptation
of the notion reducible MCSs~\citep{DBLP:conf/aaai/BrewkaE07}, provide
a suitable model for such systems.
\begin{defn}[\textbf{monotonic logic suite}]
 Let $L=(\KB,\BS,\ACC,\cost)$ be a logic suite. $L$ is \emph{monotonic}
iff
\begin{enumerate}
\item $\ACC(\kb)$ is a singleton set for every $\kb\in\KB$; and
\item $\kb\subseteq\kb^{\prime}$, $\ACC(\kb)=\{S\}$ and $\ACC(\kb^{\prime})=\{S^{\prime}\}$
implies $S\subseteq S^{\prime}$.
\end{enumerate}
\end{defn}

\begin{defn}[\textbf{definite MCS}]
 Let $M=(C_{1},\ldots,C_{n})$ be a MCS. We say that $M$ is definite iff
\begin{enumerate}
\item the logic suites $L_{1},\ldots,L_{n}$ corresponding to the contexts $C_{1},\ldots C_{n}$
are monotonic; and
\item none of the bridge rules in any context contains $\mathit{not}$.
\end{enumerate}
\end{defn}

\begin{defn}[\textbf{grounded equilibrium}]
 Let $M$ be a definite MCS. $S=(S_{1},\ldots,S_{n})$ is the \emph{grounded
equilibrium} of $M$ iff $S$ is the unique set-inclusion minimal equilibrium
of $M$.\end{defn}
\begin{rem}
Every definite MCS has exactly one unique equilibrium, which is grounded~\citep{DBLP:conf/aaai/BrewkaE07}.
\end{rem}
Unsurprisingly, the cost-complexity of reasoning for definite MCS is significantly
lower than for general MCS. The following proposition provides the first
upper estimate on the worst-case cost-complexity of the consistency problem
in definite MCS.

\begin{algorithm}
\begin{algorithmic}[1]\raggedright

\Require{a definite MCS $M=(C_{1},\ldots,C_{n})$}

\Ensure{returns the grounded equilibrium of $M$}

%\Statex

\State let $\kb_{i}^{0}\leftarrow\emptyset$ for all $i=1..n$ and $S^{0}\leftarrow(\emptyset,\ldots,\emptyset)$

\State $k\leftarrow0$

\Repeat

\State $S^{k}\leftarrow(S_{1}^{k},\ldots,S_{n}^{k})$ with $S_{i}^{k}=\ACC_{i}(\kb_{i}^{k})$

\State $\kb^{k+1}\leftarrow(\kb_{1}^{k+1},\ldots,\kb_{n}^{k+1})$ with\\
\hskip6em $\kb_{i}^{k+1}=\kb_{i}^{k}\cup\{\head r\mid r\mbox{ is satisfied in }S^{k}\}$

\State $k\leftarrow k+1$

\Until{$S^{k}\neq S^{k-1}$}

\State \textbf{return} $S^{k}$

\end{algorithmic}

\protect\caption{\label{alg:ground-equilibrium-definite} Algorithm for computing the grounded
equilibrium of a definite MCS.}
\end{algorithm}

\begin{prop}[\textbf{consistency}]
 \label{prop:definite-problems-complexity} Given a uniform-cost definite
MCS $M=(C_{1},\ldots,C_{n})$, an upper bound on the worst-case context-independent
time complexity of deciding the consistency problem for $M$, and thus also
the problems of cautious and brave reasoning w.r.t.~$M$ for some query
$p$ is 
\[
\CTime(n,m)\leq n\cdot m
\]

In the case $M$ is not a uniform-cost MCS, an upper bound on the worst-case
cost-complexity of deciding the consistency problem for $M$, we have 
\[
\cCost(n,m)\leq c\cdot\CTime(n,m)
\]
 where $ $$c=\max_{1=1..n}\cost(\ACC_{i})$.\end{prop}
\begin{proof}[Proof sketch]
 Consider Algorithm~\ref{alg:ground-equilibrium-definite}. To compute
an equilibrium (and thus decide the consistency problem), all contexts have
to be invoked at least once. After each iteration either at least one head
literal which was not true in the previous iterations becomes true and does
not cease afterwards, or no new head literal is inferred. If the latter is
the case, the process can stop. Thus, there are at most $m$ rule heads to
become true over at most $m$ steps. In every step, there are at most $n$
context being executed. Hence the upper bound.\end{proof}
\begin{cor}
\label{cor:definite-complexity-n-square} In most instances of implemented
systems the number of rules $m$ in a multi-context system will dominate
the number of contexts $n$. Hence the worst-case time-independent complexity
would typically be at most quadratic in the number of bridge rules, i.e.,
$\CTime(n,m)\leq m^{2}$ and consequently also $\cCost(n,m)\leq c\cdot m^{2}$.
\end{cor}
Now we turn our attention to the cost-complexity of reasoning in definite
MCSs. Since such MCSs have only a single unique equilibrium, the problems
of brave and cautious reasoning collapse and in turn we speak only about
a reasoning problem. It turns out, that the cost-complexity is, similarly
to the general MCS case, bound by the cost-complexity of deciding the consistency
problem, but in many practical cases can be pushed lower. Before introducing
the main result in Proposition~\ref{prop:reasoning-definite}, we first
analyse the structure of information flow leading to supporting the individual
belief sets in an equilibrium.
\begin{defn}[\textbf{fragmentary MCS}]
 We say that a MCS $M^{\prime}=(C_{1}^{\prime},\ldots,C_{n}^{\prime})$
is a \emph{MCS fragment} of another MCS $M=(C_{1},\ldots,C_{n})$ iff for
all context $C_{i}^{\prime}=(L_{i}^{\prime},\br_{i}^{\prime})$ and $C_{i}=(L_{i},\br_{i})$
with $i=1..n$, we have that $L_{i}^{\prime}=L_{i}$ and $\br_{i}^{\prime}\subseteq\br_{i}$.
We also denote $M^{\prime}\subseteq M$ and say that $M$ \emph{contains}
$M^{\prime}$. Set operations on MCS fragments are defined as the corresponding
set operations on their respective bridge rule sets.
\end{defn}

\begin{defn}[\textbf{justification}]
 \label{def:justification} Let $M=(C_{1},\ldots,C_{n})$ be a definite
MCS with a grounded equilibrium $S=(S_{1},\ldots,S_{n})$. A \emph{justification}
for a belief set $S_{i}$ in $M$ is a MCS fragment $M^{\prime}\subseteq M$,
such that
\begin{enumerate}
\item \label{enu:justification-cond} $M^{\prime}$ has a grounded equilibrium
$S^{\prime}=(S_{1}^{\prime},\ldots,S_{n}^{\prime})$, with $S_{i}^{\prime}=S_{i}$
and $S_{j}^{\prime}\subseteq S_{j}$ for every $j\neq i$; and
\item \label{enu:justification-minimality} there is no other fragment $M^{\prime\prime}\subseteq M^{\prime}$
satisfying the condition~\ref{enu:justification-cond}.
\end{enumerate}
\end{defn}

A justification of a belief set $S_{i}$ in a MCS $M$ corresponds to the
minimal set of bridge rules of $M$ which still enable derivation of $S_{i}$.
Justifications are defined w.r.t.~a given equilibrium. Support is a complementary
syntactic counterpart to the notion of justification.
\begin{defn}[\textbf{support}]
 \label{def:support} Let $M=(C_{1},\ldots,C_{n})$ be a definite MCS. The
\emph{input signature} of a context $C_{i}$ is a set of literals $\mathit{sig}(C_{i})=\{p\mid\exists r\in\br_{i}:\head r=p\}$.
An \emph{immediate} \emph{rule support} of a context $C_{i}=(L_{i},\br_{i})$
is a set of rules $\In\subseteq\br_{i}$, such that for every $p\in\mathit{sig}(C_{i})$,
there exists a rule $r\in\In$ with $\head r=p$. Finally, a \emph{support}
for a context $C_{i}=(L_{i},\br_{i})$ is a fragment $M^{\prime}=(C_{1}^{\prime},\ldots,C_{n}^{\prime})$
of $M$ with $C_{i}^{\prime}=(L_{i}^{\prime},\br_{i}^{\prime})$, such that
\begin{enumerate}
\item \label{enu:support-of-a-context} $\br_{i}^{\prime}$ contains some rule
support $\In$ of $C_{i}$;
\item \label{enu:support-of-all-rules} for every context $C_{j}^{\prime}$ literals
of which occur in a body of some bridge rule of $M^{\prime}$, $M^{\prime}$
contains also a support for $C_{j}^{\prime}$. That is, for every rule $r\in\br_{j}^{\prime}$,
if $(c_{k}:p_{k})\in\body r$, then $M^{\prime}$ contains a support of $C_{j}^{\prime}$.
Finally,
\item \label{enu:support-minimality} for every context $C_{j}^{\prime}$, we have
that $\br_{j}^{\prime}$ is minimal w.r.t.~set inclusion, such that the
conditions~\ref{enu:support-of-a-context} and~\ref{enu:support-of-all-rules}
are satisfied.
\end{enumerate}
$\mathcal{M}(C_{i})$ denotes the set of all supports of $C_{i}$. Furthermore,
$\overline{\mathcal{C}}(C_{i})=\{C_{l}^{\prime}\mid\exists M^{\prime}=(C_{1}^{\prime},\ldots,C_{n}^{\prime})\in\mathcal{M}(C_{i})\mbox{ and }\br_{l}^{\prime}\neq\emptyset\}$
and $\overline{\mathcal{R}}(C_{i})=\{r\mid\exists M^{\prime}\in\mathcal{M}(C_{i})\mbox{ and }r\in\mathcal{R}(M^{\prime})\}$
denote the sets of contexts and rules respectively \emph{supporting} $C_{i}$
in various supports in $M$.
\end{defn}
Note, there might be several immediate rule supports for a context $C_{i}$
due to possibly multiple bridge rules with the same head literal. Also, minimality
of bridge rule sets ensures that for each literal $p$, there is only a single
rule $r$ in a support $M^{\prime}$ with $p=\head r$. In turn, there might
be multiple supports for a given context in $M$. 

The following proposition relates the syntactic characterisation of sets
of rules potentially justifying a given belief set, the support, and the
sets of rules serving as an actual justification of the belief set in an
already computed belief set.
\begin{prop}
\label{prop:imports-vs-justifications} Let $M=(C_{1},\ldots,C_{n})$ be
a definite MCS with a grounded equilibrium $S=(S_{1},\ldots,S_{n})$. For
every belief set $S_{i}$ of $S$ and each of its justifications $M_{\mathit{just}}$,
there exists a support $M_{\mathit{supp}}\subseteq M$ of $C_{i}$, such
that $M_{\mathit{just}}\subseteq M_{\mathit{supp}}$ and all the bridge rules
of $M_{\mathit{just}}$ are satisfied in $S$.\end{prop}
\begin{proof}[Proof sketch]
 For a justification $M_{\mathit{just}}=(C_{1}^{j},\ldots,C_{n}^{j})$ of
a belief set $S_{i}$ with contexts $C_{i}^{j}=(L_{i}^{j},\br_{i}^{j})$,
we construct a fragment $M_{\mathit{supp}}$ of $M$ which will also be a
support of $C_{i}$ in $M$.

Firstly, for every context $C_{j}$ of $M$, either there exists an immediate
rule support $\In$ of $C_{j}$, such that $\In=\br_{j}^{\mathit{just}}$,
or we find an immediate rule support $\In$ of $C_{j}$, such that $\br_{j}^{\mathit{just}}\subseteq\In$.
Existence of such a suitable immediate rule support $\In$ of $C_{j}$ is
ensured by the minimality of $\br_{j}^{\mathit{just}}$ (c.f.~Condition~\ref{enu:justification-minimality}
of Definition~\ref{def:justification}), which ensures that for each $p\in\mathit{sig}(C_{j})$,
there's at most one rule $r$ satisfied in $M_{\mathit{just}}$ with $p=\head r$,
and the construction of immediate supports of a context, which require not
only minimality of $\In$ (c.f.~Condition~\ref{enu:support-minimality}
of Definition~\ref{def:support}), but also a full coverage of $\mathit{sig}(C_{j})$.
We construct a fragment $M_{\mathit{supp}}$ by simply extending each $\br_{j}^{\mathit{just}}$
with one of such suitable immediate rule supports. In a consequence, we have
that $M_{\mathit{just}}\subseteq M_{\mathit{supp}}$ and $M_{\mathit{supp}}$
automatically satisfies the conditions on being a support of $C_{i}$ stipulated
in Definition~\ref{def:support}.
\end{proof}
A corollary of the above proposition is that for deciding a reasoning problem
over an MCS $M$ and a query $p$, we only need to consider the contexts
and rules relevant to $p$. That is, those which support the context $C_{i}$
to which $p$ belongs, because every possible justification of a $p$-entailing
belief set of $C_{i}$ must be a subset of some support of $C_{i}$. Hence,
due to monotonicity of contexts of $M$, we can simply compute the equilibrium
of a fragment related to a union of all the possible supports of $C_{i}$
and check whether it entails $p$.
\begin{prop}
\label{prop:reasoning-definite} Let $M=(C_{1},\ldots,C_{n})$ be a uniform-cost
definite MCS and let $p$ be a query, en element of some belief set of $C_{i}$.
An upper bound on the worst-case context-independent time complexity of brave
reasoning w.r.t.~$M$ for some query $p$ is 
\[
\CTime(n,m)\leq\left|\overline{\mathcal{C}}(C_{i})\right|\cdot\left|\overline{\mathcal{R}}(C_{i})\right|\leq n\cdot m
\]

Consequently for the case where $M$ is not a uniform-cost MCS, an upper
bound on the worst-case context-independent time complexity of brave reasoning
w.r.t.~$M$ for some query $p$ is 
\[
\cCost(n,m)\leq c\cdot\CTime(n,m)\leq c\cdot\left|\overline{\mathcal{C}}(C_{i})\right|\cdot\left|\overline{\mathcal{R}}(C_{i})\right|\leq c\cdot n\cdot m
\]
 where $ $$c=\max_{C_{j}\in\overline{\mathcal{C}}(C_{i})}\cost(\ACC_{j})$.
\end{prop}

\begin{cor}
Similarly to the observation in Corollary~\ref{cor:definite-complexity-n-square},
in most instances of implemented systems the number of rules $m$ in supports
of a given context $C_{i}$ will dominate the number of contexts $n$, the
worst-case time-independent complexity of reasoning problem would typically
be at most quadratic in terms of the number of rules of the relevant support.
In the worst-case, though, a support can include all the bridge rules of
the original MCS.
\end{cor}

\section{Acyclic definite MCS \label{sec:Acyclic-definite-MCS}}

Now we turn our attention to multi-context systems which do not contain cycles
in the information flow structure induced by their bridge rules. While relatively
simplistic, the class of acyclic MCSs is practically important. Many knowledge-intensive
systems-of-systems and information-aggregation applications feature a hierarchical
structure with raw information sources at the bottom and gradually more and
more abstract and higher-level information-processing components towards
the top. The hierarchical structure of such systems is dictated by the fact,
that such systems capture the knowledge of human experts in a given domain.
Typically, the structure of domain knowledge articulated by such experts
tends to be relatively simple and hierarchical too, as such structures are
easier to understand and manipulate for humans.
\begin{defn}[\textbf{stratified definite MCS}]
 Let $M=(C_{1},\ldots,C_{n})$ be a definite MCS with contexts $\mathcal{C}(M)=\{C_{1},\ldots,C_{n}\}$.
A decomposition $\Strat=\Strat_{0},\ldots,\Strat_{m}$ of $M$ with $\Strat_{i}\subseteq\mathcal{C}(M)$
and $\Strat_{k}\cap\Strat_{l}=\emptyset$ for every $k,l=1..n$ is a \emph{stratification}
of $M$ iff for each bridge rule $r\in\br_{i}$ of a context $C_{i}\in\Strat_{k}$
each of its body literals $(j:p)\in\body r$ corresponds to a context $C_{j}\in\Strat_{l}$
with $l<k$.

We say that a stratification $\Strat=\Strat_{0},\ldots,\Strat_{m}$ of a
MCS $M$ is \emph{compact }iff there is no other stratification $\Strat^{\prime}=\Strat_{0}^{\prime},\ldots,\Strat_{m^{\prime}}^{\prime}$
of $M$, such that there is a context $C$ of $M$ with $C\in\Strat_{i}$
and $C\in\Strat_{j}^{\prime}$, while at the same time $j<i$.

A definite MCS $M$ is said to be \emph{stratified}, or \emph{acyclic}, iff
there exists a stratification of $M$.
\end{defn}

In stratified definite MCSs the information flows unidirectionally from contexts
without any bridge rules in the stratum $\mathfrak{S}_{0}$ (information
sources), upwards to higher-level contexts up to those in the top-most stratum
for which there are no bridge rules in the MCS containing elements of their
belief states in their respective bodies. Thanks to stratification, we also
straightforwardly have that all the contexts involved in supports of a given
context $C$ belong to lower strata than $C$ does and at the same time there
are no bridge-rule dependencies among the context of a single stratum. This
insight leads to the following cost complexity result.
\begin{prop}
\label{prop:definite-stratified-problems-complexity} Given a uniform-cost
stratified definite MCS $M=(C_{1},\ldots,C_{n})$, an upper bound on the
worst-case context-independent time complexity of deciding the consistency
problem for $M$, as well as the reasoning problem w.r.t.~$M$ for some
query $p$ is 
\[
\CTime(n,m)\leq n
\]

In the case $M$ is not a uniform-cost MCS, an upper bound on the worst-case
cost-complexity of deciding the consistency problem for $M$, as well as
the reasoning problem w.r.t.~$M$ for some query $p$, we have 
\[
\cCost(n,m)\leq c\cdot\CTime(n,m)=c\cdot n
\]
 where $ $$c=\max_{i=1..n}\cost(\ACC_{i})$.
\end{prop}
\begin{algorithm}
\begin{algorithmic}[1]\raggedright

\Require{a stratified definite MCS $M=(C_{1},\ldots,C_{n})$}

\Ensure{returns $S$, the grounded equilibrium of $M$}

%\Statex

\State compute $\mathfrak{S}=\mathfrak{S}_{0},\ldots,\mathfrak{S}_{l}$,
the compact stratification of $M$

\State $S\leftarrow(S_{1},\ldots,S_{n})$ with $S_{i}\leftarrow\emptyset$

\For{$k=1..l$}

\For{each $C_{i}\in\Strat_{k}$}

\State $\kb_{i}\leftarrow\{\head r\mid r\in\br_{i}\mbox{ is applicable in }S\}$

\State update $S$ with $S_{i}=\ACC_{i}(\kb_{i})$

\EndFor

\EndFor

\State \textbf{return} $S$

\end{algorithmic}

\protect\caption{\label{alg:consistency-stratified} Algorithm computing the grounded equilibrium
in a definite stratified MCS.}
\end{algorithm}

\begin{proof}
Algorithm~\ref{alg:consistency-stratified} computes the grounded equilibrium
of a MCS $M$, hence it decides the consistency problem of $M$. Its soundness
and completeness follows inductively from the following induction hypothesis:
Let $S_{c_{1}},\ldots,S_{c_{k}}$ be belief sets of contexts $C_{c_{1}},\ldots,C_{c_{k}}$
involved in body literals of a union of all immediate rule supports of a
context $C_{i}$ in a stratum $\mathfrak{S}_{k}$. Given a knowledge base
$\kb_{i}=\{\head r\mid r\in\br_{i}\textrm{ and }r\textrm{ is satisfied w.r.t. }S_{c_{1}},\ldots,S_{c_{k}}\}$,
$\ACC_{i}(\kb_{i})=\{S_{i}\}$ is a singleton belief set $C_{i}$ corresponding
to the grounded equilibrium of $M$. The algorithm evaluates every context
of $M$ exactly once, hence its worst-case cost-complexity equals the number
of contexts of $M$.
\end{proof}

\section{Incremental reasoning in definite MCS}

The discourse in previous sections followed the structure of standard results
from logic programming reflected in the framework of multi-context systems.
Generally, the cost-complexity results followed the known results on computational
time complexity of computation of models of logic programs. In this section,
we build upon the concepts introduced above and introduce an incremental
approach for computing equilibria of cost-aware multi-context systems, which
specifically focuses on reduction of cost-complexity of such computation. 

By exploiting the structure of a given MCS with respect to a given query,
the actual cost-complexity of solving the reasoning problem can be often
improved upon. Before introducing an actual algorithm reducing the number
of context evaluations while computing a solution to a reasoning problem,
let us first introduce a notion of a fragment depending on a context. The
concept is complementary to the context support (Definition~\ref{def:support}),
but instead of considering contexts and rules necessary for deriving a belief
set of a context $C_{i}$, a \emph{fragment depending on a context} is a
fragmentary MCS in which computation of belief sets of all other contexts
is influenced, \emph{depends }on the context $C_{i}$. In other words, change
in a belief set of $C_{i}$ can potentially enforce a change in the input
knowledge base and thus also a change in the output belief set of the depending
contexts.
\begin{defn}[\textbf{fragment depending on a context}]
 Let $M=(C_{1},\ldots,C_{n})$ be a definite MCS. We say that a fragment
$M^{\prime}=(C_{1}^{\prime},\ldots,C_{n}^{\prime})$ of $M$ \emph{depends}
on a context $C_{i}$ of $M$ iff for all contexts $C_{j}^{\prime}$ we have
that $(i:p)\in\body r$ with $r\in\br_{j}$ implies that $r\in\br_{j}^{\prime}$
and at the same time $M^{\prime}$ contains a fragment $M^{\prime\prime}$
\emph{depending} on $C_{j}^{\prime}$. 

Given a set of contexts $\mathcal{C}=C_{i_{1}},\ldots,C_{i_{k}}$, we say
that $M^{\prime}$ \emph{depends precisely }on $\mathcal{C}$ iff for each
context $C_{i_{j}}\in\mathcal{C}$ it contains the fragment $ $of $M$ depending
on $C_{i_{j}}$, while at the same time $M^{\prime}$is minimal such w.r.t.~set
inclusion on the sets of bridge rules.

For convenience, we define an \emph{empty fragment} as a fragment with all
bridge rule sets empty. We also say, that a context $C=(L,\br)$ is \emph{valid}
iff $\br\neq\emptyset$.
\end{defn}

Proposition below provides an insight into computation required for ``extending''
a grounded equilibrium of a MCS fragment to a grounded equilibrium of its
extension by another fragment.
\begin{prop}
\label{prop:fragments-equilibria} Let $M=(C_{1},\ldots,C_{n})$ be a definite
MCS and $M^{\prime}$ be its fragment depending on a context $C_{i}$ with
a grounded equilibrium $S^{\prime}=(S_{1}^{\prime},\ldots,S_{n}^{\prime})$.
Let also $M_{\mathit{supp}}$ be a fragment of $M$ depending precisely on
the set of contexts valid in $M_{\mathit{diff}}=M\setminus M^{\prime}$.
Recall, the set of bridge rules corresponding to the $i$-th context of $M_{\mathit{diff}}$
is defined as $\br_{i}^{\mathit{diff}}=\br_{i}^{\prime}\setminus\br_{i}$
and $M_{\mathit{supp}}$ is depends precisely on the contexts $C_{i_{1}},\ldots,C_{i_{k}}$,
such that $\br_{i_{k}}^{\mathit{diff}}\neq\emptyset$.

The grounded equilibrium $S=(S_{1},\ldots,S_{n})$ of $M$ can be computed
as a union $S=S^{\prime}\cup S_{\mathit{supp}}$, where $S_{\mathit{supp}}$
is a grounded equilibrium of $M_{\mathit{supp}}$.\end{prop}
\begin{proof}[Proof sketch]
 $M$ can be decomposed into two fragments disjoint w.r.t.~their bridge
rules: $M^{\prime}$ and $M\setminus M^{\prime}$. In turn, for every context
$C_{i}$ of $M$, we have that either 
\begin{enumerate}
\item \label{enu:incr-cond-1} its projection in $M\setminus M^{\prime}$ is a
valid context, i.e., there exists at least one bridge rule in its projection
in $M\setminus M^{\prime}$; or 
\item its projection in $M\setminus M^{\prime}$ features an empty set of bridge
rules, but the context depends on another context for which \ref{enu:incr-cond-1}
is the case, or
\item all its bridge rules were already contained in $M^{\prime}$ and the same
holds for all the contexts it is supported by, depends on. 
\end{enumerate}
In the first case, the belief set of the context $C_{i}$ needs to be recomputed,
since a new literal needs to be possibly added to the context's knowledge
base. Due to monotonicity of logic suites in a definite MCS, it is ensured
that the resulting belief set will be a (non-strict) superset of the original
one. Should that be the case, also all the contexts depending on it need
to be recomputed as well, regardless whether their bridge rules completely
belonged to $M^{\prime}$, or not. In the second case, by the same argument
the context's belief set needs to be recomputed as well. Finally, when the
context's bridge rules completely belonged to $M^{\prime}$ and it does not
depend on any context which needs to be recomputed, its belief set equals
the corresponding belief set in the grounded equilibrium of $M$.
\end{proof}

\begin{cor}
\label{cor:fragments-union-equilibria} Let $M_{1},M_{2}\subseteq M$ be
fragments of a MCS $M$. A grounded equilibrium $S_{1}$ of $M_{1}$ can
be extended to a grounded equilibrium $S_{1,2}$ of a union $M_{1,2}=M_{1}\cup M_{2}$
as $S_{1,2}=S_{1}\cup S_{\mathit{add}}$, where $S_{\mathit{add}}$ is a
grounded equilibrium of a fragment of $ $$M_{1,2}$ depending precisely
on the contexts valid in $M_{1}\setminus M_{2}$.
\end{cor}
Corollary~\ref{cor:fragments-union-equilibria} is a straightforward reformulation
of Proposition~\ref{prop:fragments-equilibria}. We conclude the discourse
by exposing Algorithm~\ref{alg:reasoning-stratified} for incremental reasoning
in definite MCSs and exploiting the corollary. It first computes all the
supports of the context the query corresponds to and then iteratively selects
them one by one and incrementally constructs the fragmentary grounded equilibrium
of the input MCS. If during the computation $p$ is derived as an equilibrium
of some of the fragments, due to monotonicity of logic suites in a definite
MCS, by necessity, the equilibrium of $M$ must also entail $p$. 

Applied to stratified definite MCS, the algorithm could be further improved
by selecting the cheapest fragment of $\mathcal{M}$. That is, one with the
lowest equilibrium computation estimate for its corresponding $M_{\mathit{supp}}$
fragment. In stratified definite MCSs that cost corresponds to the sum of
costs of the contexts which need to be recomputed.

Note, the applicability of the algorithm is constrained to MCS information-flow
structures with relatively small overlaps between different supports of the
context deriving the query answer. In presence of high redundancy among contexts,
that is almost all contexts depend on almost all the other, the algorithm
will recompute the whole MCS too often, though. This is, however, seldom
the case in implemented systems, such as the maritime traffic surveillance
system \noun{Metis}~\citep{Hendriks2013542}.

\begin{algorithm}
\begin{algorithmic}[1]\raggedright

\Require{a definite MCS $M=(C_{1},\ldots,C_{n})$ and a query $p$ corresponding
to a context $C_{i}$}

\Ensure{returns true iff the grounded equilibrium of $M$ entails $p$}

%\Statex

\State compute the set $\mathcal{M}$ of fragments of $M$ supporting $C_{i}$

\State $M_{\mathit{done}}\leftarrow$ an empty fragment of $M$

\State $S\leftarrow(S_{1},\ldots,S_{n})$ with $S_{i}=\emptyset$

\Repeat

\State select $M^{\prime}\in\mathcal{M}$

\State $\mathcal{M}\leftarrow\mathcal{M}\setminus\{M^{\prime}\}$

\State construct $M_{\mathit{supp}}$, a fragment of $M_{\mathit{done}}\cup M^{\prime}$
precisely\\
\hskip6em depending on the valid contexts of $M^{\prime}\setminus M_{\mathit{done}}$

\State $\mathcal{S}_{\mathit{supp}}\leftarrow$compute the equilibrium of
$M_{\mathit{supp}}$

\State $S\leftarrow S\cup S_{\mathit{supp}}$

\State $M_{\mathit{done}}\leftarrow M_{\mathit{done}}\cup M^{\prime}$

\State\textbf{if} $S$ entails $p$ \textbf{then} \textbf{return} true

\Until{$ $$\mathcal{M}=\emptyset$}

\State \textbf{return} false

\end{algorithmic}

\protect\caption{\label{alg:reasoning-stratified} Incremental reasoning for  definite MCSs.}
\end{algorithm}

\section{Discussion and final remarks}

The motivating premise underlying the presented work is that MCSs are a suitable
model for design, implementation and analysis of deployed knowledge-intensive
systems-of-systems featuring heterogeneous components. Here, we extend the
generic model of multi-context systems with the notion of a cost of executing
semantic operators of the individual contexts. The idea is to facilitate
system scalability in terms of the incurred costs, be it computational costs,
bandwidth, or even financial expenses. Our focus on cost-complexity, rather
than computational-complexity of algorithms also presents a novel view on
design and deployment of information-aggregation and reasoning systems.

We introduced a series of worst-case complexity results for general, definite
and stratified MCSs. Some of these results could be still improved upon by
taking inspiration from e.g.,~\citep{DBLP:conf/jelia/BairakdarDEFK10},
where the authors perform an ear-decomposition of possibly cyclic MCSs, so
as to streamline distributed computation of their equilibria. Another inspiration
could be to exploit the results stemming from the results by \citet{DBLP:conf/aaai/GottlobPW06}
and analyse the information-flow graph induced by bridge rules of an MCS,
in order to exploit the results relating the time-complexity of computation
with such systems to the tree-width of the information-flow graph.

We already implemented the incremental algorithm for deciding reasoning problems
in the context of our work on \noun{Metis}, a system-of-systems aiming at
maritime traffic surveillance and risk assessment of ships sailing in busy
coastal waters, such as in the Dutch Exclusive Economic Zone. We describe
the system and its reconfiguration component based on ideas described above
in a separate submission~\citep{METISECAI2014}.

\section*{Acknowledgements}

This work was supported by the Dutch national program \textsf{COMMIT}. The
research was carried out as a part of the \noun{Metis} project under the
responsibility of the \emph{TNO-Embedded Systems Innovation},\emph{ }with
\emph{Thales Nederland B.V.}~as the carrying industrial partner.

\bibliographystyle{aaai}
\bibliography{/home/pno/academic/bibliography/logics,/home/pno/academic/bibliography/metis,/home/pno/academic/bibliography/pno,/home/pno/academic/bibliography/asp}

\end{document}